\documentclass[11pt,twoside]{article}

\usepackage{fullpage}

\usepackage{epsf}
\usepackage{fancyheadings}
\usepackage{graphics}
\usepackage{graphicx}
\usepackage{psfrag}

\usepackage{color}

\usepackage{amsthm}
\usepackage{amsfonts}
\usepackage{amsmath}
\usepackage{amssymb}
\usepackage[linesnumbered, ruled, vlined]{algorithm2e}

\theoremstyle{plain}

\newtheorem{theorem}{Theorem}

\newtheorem{lemma}{Lemma}
\newtheorem{corollary}{Corollary}
\newtheorem{definition}{Definition}

\newlength{\widebarargwidth}
\newlength{\widebarargheight}
\newlength{\widebarargdepth}

\makeatletter
\long\def\@makecaption#1#2{
        \vskip 0.8ex
        \setbox\@tempboxa\hbox{\small {\bf #1:} #2}
        \parindent 1.5em  
        \dimen0=\hsize
        \advance\dimen0 by -3em
        \ifdim \wd\@tempboxa >\dimen0
                \hbox to \hsize{
                        \parindent 0em
                        \hfil 
                        \parbox{\dimen0}{\def\baselinestretch{0.96}\small
                                {\bf #1.} #2
                                } 
                        \hfil}
        \else \hbox to \hsize{\hfil \box\@tempboxa \hfil}
        \fi
        }
\makeatother

\long\def\comment#1{}

\newcommand{\matsnorm}[2]{|\!|\!| #1 | \! | \!|_{{#2}}}

\newcommand{\dH}{\mathsf{d}_{\mathsf{H}} }

\newcommand{\nucnorm}[1]{\ensuremath{\matsnorm{#1}{\tiny{\mbox{nuc}}}}}
\newcommand{\fronorm}[1]{\frobnorm{#1}}
\newcommand{\frobnorm}[1]{\ensuremath{\matsnorm{#1}{\mbox{\tiny{F}}}}}
\newcommand{\opnorm}[1]{\ensuremath{\matsnorm{#1}{\tiny{\mbox{op}}}}}

\newcommand{\Exs}{\ensuremath{{\mathbb{E}}}}

\newcommand{\numobs}{\ensuremath{n}}
\newcommand{\usedim}{\ensuremath{d}}

\newcommand{\1}{\ensuremath{{\sf (i)}}}
\newcommand{\2}{\ensuremath{{\sf (ii)}}}
\newcommand{\3}{\ensuremath{{\sf (iii)}}}
\newcommand{\4}{\ensuremath{{\sf (iv)}}}

\DeclareMathOperator{\diag}{diag}

\DeclareMathOperator{\rank}{{\sf rank}}

\DeclareMathOperator{\range}{{\sf range}}

\newcommand{\NORMAL}{\ensuremath{\mathcal{N}}}

\newcommand{\thetastar}{\ensuremath{\theta^*}}
\newcommand{\thetahat}{\ensuremath{\widehat{\theta}}}

\newcommand{\xhat}{\ensuremath{\widehat{x}}}
\newcommand{\Deltahat}{\ensuremath{\widehat{\Delta}}}
\newcommand{\udiff}{\ensuremath{\mathbb{U}^{{\sf diff}}_m(A)}}
\newcommand{\Xhat}{\ensuremath{\widehat{X}}}
\newcommand{\Yhat}{\ensuremath{\widehat{Y}}}
\newcommand{\Yhatsrl}{\ensuremath{\widehat{Y}_{{\sf sr}}(\lambda_0)}}

\newcommand{\Xhatml}{\ensuremath{\widehat{X}_{{\sf ML}}}}

\newcommand{\deltanm}{\ensuremath{\delta_{n,m}}}

\newcommand{\widgraph}[2]{\includegraphics[keepaspectratio,width=#1]{#2}}

\newcommand{\real}{\ensuremath{\mathbb{R}}}

\newcommand{\Pihat}{\ensuremath{\widehat{\Pi}}}
\newcommand{\Pihatml}{\ensuremath{\widehat{\Pi}_{{\sf ML}}}}

\newcommand{\EE}{\ensuremath{\mathbb{E}}}

\theoremstyle{definition}
\newtheorem{example}{Example}

\newcommand{\MSE}{\ensuremath{\mathcal{R}}}
\newcommand{\Pistar}{\ensuremath{\Pi^*}}
\newcommand{\PERMN}{\ensuremath{\mathcal{P}_\numobs}}
\newcommand{\Xstar}{\ensuremath{X^*}}
\newcommand{\defn}{\ensuremath{: \, = }}
\newcommand{\mdim}{\ensuremath{m}}

\newcommand{\PiLev}{\ensuremath{\Pihat_{{\sf lev}}}}
\newcommand{\XLev}{\ensuremath{\Xhat_{{\sf lev}}}}

\newcommand{\tracer}[2]{\ensuremath{\langle \!\langle {#1}, \; {#2}
\rangle \!\rangle}}

\begin{document}


\begin{center}

{\bf{\LARGE{Denoising Linear Models with Permuted Data}}}

\vspace*{.2in}

{\large{
\begin{tabular}{ccc}
Ashwin Pananjady$^\dagger$ &  Martin J. Wainwright$^{\dagger,\star}$ & Thomas A. Courtade$^\dagger$ \\
\end{tabular}
}}

\vspace*{.2in}

\begin{tabular}{c}
Department of Electrical Engineering and Computer Sciences$^\dagger$ \\
Department of Statistics$^\star$ \\
UC Berkeley
\end{tabular}

\vspace*{.2in}

\today

\vspace*{.2in}

\begin{abstract}
The multivariate linear regression model with shuffled data and
additive Gaussian noise arises in various correspondence estimation
and matching problems. Focusing on the denoising aspect of this
problem, we provide a characterization the minimax error rate that is
sharp up to logarithmic factors. We also analyze the performance of
two versions of a computationally efficient estimator, and establish
their consistency for a large range of input parameters. Finally, we
provide an exact algorithm for the noiseless problem and demonstrate
its performance on an image point-cloud matching task. Our analysis
also extends to datasets with outliers.
\end{abstract}
\end{center}


\section{Introduction}
\label{sec:introduction}

The linear model is a ubiquitous and well-studied tool for predicting
responses $y$ based on a vector $a$ of covariates or predictors. In
this paper, we consider the multivariate version of the model, with
vector-valued responses $y_i \in \real^m$, and covariates $a_i \in
\real^d$.  In the standard formulation of this problem, estimation is
performed on the basis of a data set of $\numobs$ pairs $\{a_i, y_{i}
\}_{i=1}^\numobs$, in which each response $y_i$ is correctly
associated with the covariate vector $a_i$ that generated it.  Our
focus is instead on the following variant of the standard set-up:
the input consists of the permuted data set $\{a_i,
y_{\pi_i}\}_{i=1}^n$, where $\pi$ represents an unknown permutation.
The presence of this unknown permutation---which can be viewed as a
nuisance parameter---introduces substantial challenges to this
problem.

It is convenient to introduce matrix-vector notation so as to state
the problem more precisely.  If we form the matrices $A \in
\real^{\numobs \times \usedim}$ and $Y \in \real^{\numobs \times
  \mdim}$ with $a_i^T$ and $y_i^T$, respectively, as their $i^{th}$
row, we arrive at the model
\begin{align}
 \label{obsmodel}
Y & = \Pi^* A X^* + W,
\end{align}
where $\Pi^*$ is an unknown $n \times n$ permutation matrix, $X^* \in
\real^{d \times m}$ is an unknown matrix of parameters, and $W$ is the
additive observation noise\footnote{We refer to the setting $W= 0$
  \emph{a.s.} as the noiseless case.}. When $m = 1$, this reduces to
the vector linear regression model with an unknown permutation, given
by
\begin{align}
y = \Pi^* A x^* + w, \label{shuffvec}
\end{align}
which we refer to as the shuffled vector model.

The observation model~\eqref{obsmodel} arises in multiple
applications, which are discussed in detail for the shuffled vector
model~\eqref{shuffvec} in our earlier work \cite{permpaper}. Here let
us describe two applications that arise in the multivariate setting
($m > 1$), which we use as running examples throughout the paper.

\begin{figure}[h]
    \begin{center}
    \vspace{-2mm} \includegraphics[scale=.45]{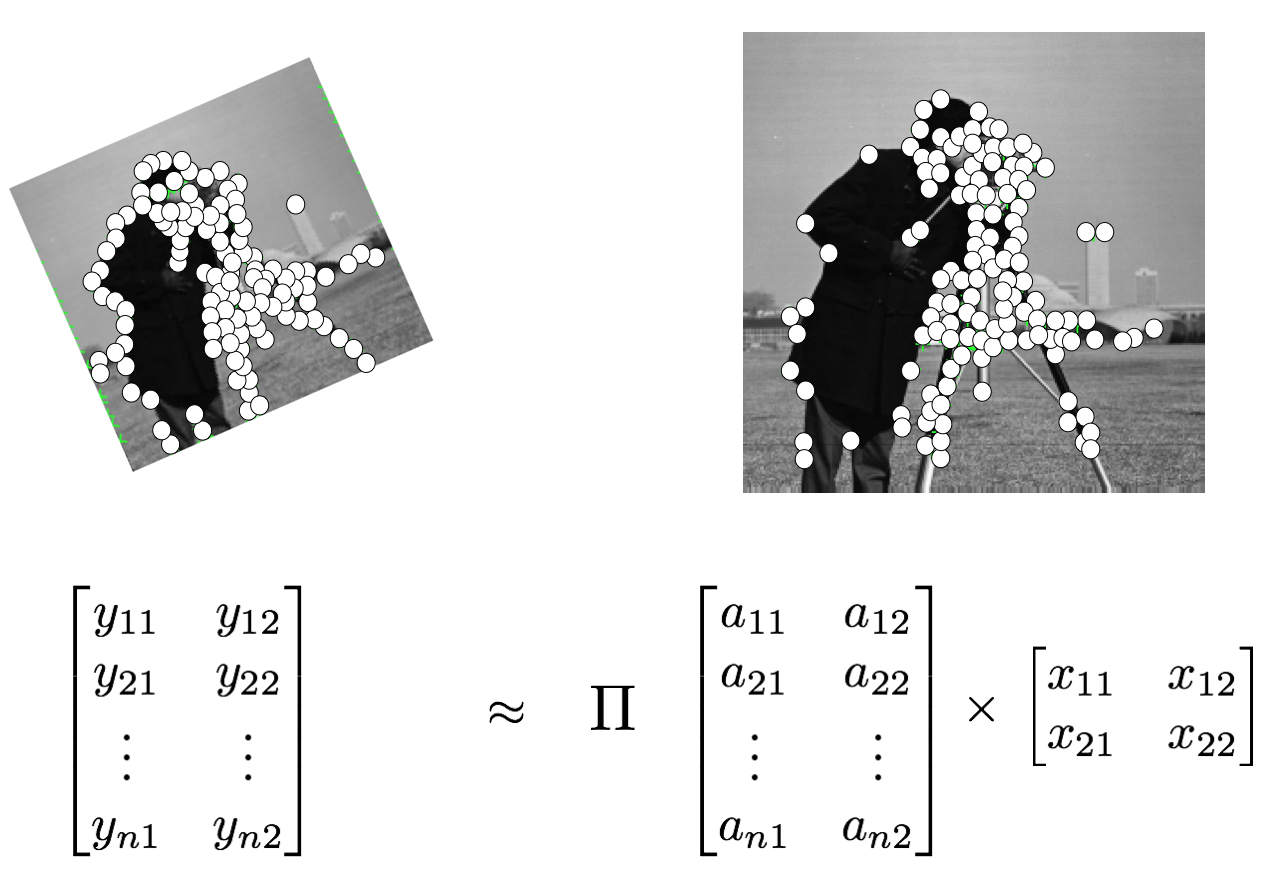}
    \caption{Example of pose and correspondence estimation for 2D
      images. The image coordinates are related by an unknown resizing
      and rotation $X$. The unknown permutation represents the
      correspondence between keypoints (white circles) obtained via
      corner-detection. The matrices $Y$ and $A$ represent coordinates
      of all keypoints, and approximately obey the relation
      \eqref{obsmodel} because all the keypoints detected in the two
      images are not the same. \vspace{-.4cm}} \label{fig:pose}
    \end{center}
\end{figure}

\begin{example}[Pose and correspondence estimation]
Our first motivating application is the problem of pose and
correspondence estimation in images~\cite{marques2009subspace}; it is
closely related to point-cloud matching in
graphics~\cite{mann1993compositing}. Suppose that we are given two
images of a similar object, with the coordinates of one image arising
from an unknown linear transformation of the coordinates of the
second. In order to determine the linear transformation, keypoints are
detected in each of the images individually and then matched; see
Figure~\ref{fig:pose} for an illustration. We emphasize that in
practice, the keypoint detection algorithm also returns features that
help in finding the matching permutation $\Pi^*$, but our goal here is
to analyze whether there are procedures that are robust to such
features being missing or corrupted.  It is also worth noting that
while in this example we have $d = m = 2$, the model is also valid for
higher (but equal) parameters $d$ and $m$, if we assume that in
addition to the coordinates of the keypoints, other attributes like
pixel brightness, colour, etc. in the two images are also related by a
linear transformation.
\end{example}

\begin{example}[Header-free communication]
A second application is that of header-free communication in large
communication networks~\cite{permpaper}. Suppose that we use multiple
sensors to take noisy measurements of a unknown matrix $X^*$ of
parameters; each measurement corresponds to a noisy linear observation
of the form $a_i^\top X^* + w_i^\top$. In very large networks, such as
those that arise in Internet of Things applications, it is often found
that the bandwidth between a sensor and fusion center is mainly
dominated by a header containing identity information---that is, by a
bitstring that identifies sensor $i$ to the fusion
center~\cite{keller2009identity}. One possible solution to this
problem is header-free communication, meaning that the identities of
the sensors that sent the signal are no longer known to the fusion
center.  This absence can be modeled by introducing the unknown
permutation matrix as in our model. If we are still able to achieve
similar statistical performance without these headers, then such an
approach is clearly preferable from a bandwidth standpoint.
\end{example}

\vspace*{.05in}

With this motivation in hand, let us now provide a high-level overview
of the main results of this paper.  We focus on the multivariate
model~\eqref{obsmodel} with a fixed design matrix $A$, and
Gaussian\footnote{Our results also extend to the case of
  i.i.d. sub-Gaussian noise.} noise $W_{ij} \stackrel{i.i.d.}{\sim}
\NORMAL(0, \sigma^2)$. We evaluate an estimator $(\Pihat, \Xhat)$
based on its ``denoising" capability, which we capture using the
normalized prediction error $\frac{1}{nm} \frobnorm{\Pihat A \Xhat -
  \Pi^* A X^* }^2$. Our primary objective in this paper is to
characterize the fundamental limits of denoising in a minimax sense.
In particular, an estimator is any measurable mapping of the input
$(y, A)$ to estimates $(\Pihat, \Xhat)$ of the permutation and
regression matrix, and we measure the quality of these estimates via
their uniform mean-squared error
\begin{subequations}
\begin{align}
\MSE(\Pihat, \Xhat) & \defn \frac{1}{\numobs \mdim} \sup_{
  \substack{\Pistar \in \PERMN \\ \Xstar \in \real^{\usedim \times
      \mdim}}} \Exs \frobnorm{\Pihat A \Xhat - \Pi^* A X^* }^2,
\end{align}
where the expectation is taken over the noise $W$, and any randomness
in the estimator $(\Pihat, \Xhat)$.  Note that a control on this
quantity ensures that the estimator $(\Pihat, \Xhat)$ performs
uniformly well over the full class of permutation and regression
matrices.  By taking an infimum over all estimators, we arrive at the
minimax risk associated with the problem, viz.
\begin{align}
  \label{eq:defminimax}
\inf_{\substack{\Pihat \in \mathcal{P}_n \\ \Xhat \in \real^{d \times
      m}}} \MSE(\Pihat, \Xhat) \; = \; \inf_{\substack{\Pihat \in
    \mathcal{P}_n \\ \Xhat \in \real^{d \times m}}} \frac{1}{\numobs
  \mdim} \sup_{ \substack{\Pistar \in \PERMN \\ \Xstar \in
    \real^{\usedim \times \mdim}}} \Exs \frobnorm{\Pihat A \Xhat -
  \Pi^* A X^* }^2.
\end{align}
\end{subequations}
Our interest will be in upper and lower bounding this quantity as a
function of the design matrix $A$, dimensions $(\numobs, \mdim,
\usedim)$ and the noise variance $\sigma^2$.  We also demonstrate an
explicit (but computationally expensive) algorithm that achieves the
minimax risk up to a $\log(\numobs)$ factor, and analyze
polynomial-time estimators with slightly larger prediction error.

In both of the examples discussed above, estimators with small minimax
prediction error are of interest. In the pose and correspondence
estimation problem, obtaining low prediction error is equivalent to
obtaining near-identical keypoint locations on both images; in the
sensor network example, we are interested in obtaining a set of
noise-free linear functions of the input signal.  It is important to
note that depending on the application, multiple regimes of the
parameter triplet $(\numobs, \mdim, \usedim)$ are of
interest. Therefore, in this paper, we focus on capturing the
dependence of denoising error rates on all of these parameters, and
also on the structure of the matrix $A$.

Our work contributes to the growing body of literature on regression
problems with unknown permutations, as well as related row-space
perturbation problems including blind
deconvolution~\cite{ling2015self}, phase
retrieval~\cite{candes2015phase}, and dictionary
learning~\cite{tosic2011dictionary}.  Regression problems with unknown
permutations have been considered in the context of statistical
seriation and univariate isotonic matrix recovery~\cite{rigollet}, and
non-parametric ranking from pairwise comparisons~\cite{nihar}, which
involves bivariate isotonic matrix recovery. Moreover, the prediction
error is used to evaluate estimators in both these applications.

Specializing to our setting, the shuffled vector
model~\eqref{shuffvec} was first considered in the context of
compressive sensing with a sensor
permutation~\cite{emiya2014compressed}. The first theoretical results
were provided by Unnikrishnan et al.~\cite{unl}, who provided
necessary and sufficient conditions needed to recover an adversarially
chosen $x^*$ in the noiseless model with a random design matrix
$A$. Also in the random design setting, our own previous
work~\cite{permpaper} focused on the complementary problem of
recovering $\Pi^*$ in the noisy model, and showed necessary and
sufficient conditions on the SNR under which exact and approximate
recovery were possible. An efficient algorithm to compute the maximum
likelihood estimate was also provided for the special case $d= 1$.


\subsection{Our contributions}

First, we characterize the minimax prediction error of multivariate
linear model with an unknown permutation up to a logarithmic factor,
by analyzing the maximum likelihood estimator. Since the maximum
likelihood estimate is NP-hard to compute in general~\cite{permpaper},
we then propose a computationally efficient estimator based on
singular value thresholding and sharply characterize its performance,
showing that it achieves vanishing prediction error over a restricted
range of parameters. We also propose a variant of this estimator that
achieves the same error rates, but with the advantage that it does not
require the noise variance to be known. Third, we propose an efficient
spectral algorithm for the noiseless problem that is exact provided
certain natural conditions are met. We demonstrate this algorithm on
an image point cloud matching task. Finally, we extend our results to
a richer class of models that allows for outliers in the dataset.  In
the next section, we collect our main theorems and discuss their
consequences. Proofs are postponed to Section~\ref{sec:proofs}.

\paragraph{Notation:}
We use $\PERMN$ to denote the set of permutation matrices. Let
$I_\usedim$ denote the identity matrix of dimension $\usedim$.  We
use the notation $\frobnorm{M}$, $\opnorm{M}$, and $\nucnorm{M}$ to
denote the Frobenius, operator, and nuclear norms of a matrix $M$, and
$c, c_1, c_2$ to denote universal constants that may change from line
to line.


\section{Main results}
\label{sec:main}

In this section, we state our main results and discuss some of their
consequences.  We divide our results into four subsections, having to
do with minimax rates, polynomial time estimators, efficient
procedures for the noiseless problem, and an extension of the
model~\eqref{obsmodel} that allows for outliers.


\subsection{Minimax rates of prediction}
\label{sec:minimax}

Assuming that the noise $W$ is i.i.d. Gaussian, so the maximum
likelihood estimate (MLE) of the parameters $(\Pi^*, X^*)$ is given by
\begin{align}
 \label{eq:ml}
(\Pihatml, \widehat{X}_{{\sf ML}}) = \arg \min_{\substack{\Pi \in
     \PERMN \\ X \in \real^{d \times m}}} \frobnorm{Y - \Pi A X}^2.
\end{align}
This estimator is also sensible for non-Gaussian noise, as long as
its tail behavior is similar to the Gaussian case (as can be formalized
by the notion of sub-Gaussianity).

In this section, we begin by providing an upper bound the prediction
error achieved by the maximum likelihood estimator for any design
matrix $A$.  In general, however, it is impossible to prove a matching
lower bound for an arbitrary matrix $A$.  As an extreme example,
suppose that the matrix $A$ with identical rows: in this case, the
permutation matrix $\Pi^*$ plays no role whatsoever, and the problem
is obviously much easier than with a generic matrix $A$.

With this fact in mind, we derive lower bounds that apply provided the
matrix $A$ lies in a restricted class, in order to define which we
require some additional notation. For a vector $v$, let $v^s$ denote
the vector sorted in decreasing order, and let $\mathbb{B}_{2, n}(1)$
denote the $n$-dimensional $\ell_2$-ball of unit radius centered at
$0$. Define the matrix class
\begin{align}
\mathcal{A}(\gamma, \xi) &= \Big\{A \in \real^{n \times d} \mid
\exists a \in \range(A) \cap \mathbb{B}_{2,n}(1) \text{ with }
a^s_{\lfloor \gamma n \rfloor} \geq a^s_{\lfloor \gamma n \rfloor + 1}
+ \xi \Big\}. \nonumber
\end{align}
In rough terms, this condition defines matrices that are not ``flat'',
meaning that there is some vector in their range obeying the $(\gamma,
\xi)$-separation condition defined above. It can be verified that a
matrix $A$ with i.i.d. sub-Gaussian entries lies in the class
$\mathcal{A}(C_1, C_2/\sqrt{n})$ with high probability for fixed
constants $C_1, C_2$. We are now ready to state our first main result:
\begin{theorem}
\label{thm:MLrate}
For any triple $(A, X^*, \Pi^*) \in \real^{\numobs \times \usedim}
\times \real^{\usedim \times \mdim} \times \PERMN$, we have
\begin{subequations}
\begin{align}
 \frac{ \frobnorm{\Pihatml A \Xhatml - \Pi^* A X^* }^2}{nm} & \leq c_1
 \sigma^2 \left( \frac{\rank(A)}{n} + \frac{1}{m} \min\left\{\log n, m
 \right\} \right), \label{upperbound}
\end{align}
with probability greater than $1 - e^{-c(n\log n + m \rank(A))}$.

Conversely, for any matrix $A \in \mathcal{A}(C_1, C_2/\sqrt{n})$, and
any estimator $(\Pihat, \Xhat)$, we have
\begin{align}
  \sup_{\substack{\Pi^* \in \mathcal{P}_n \\ X^* \in \real^{d \times
        m}}} \EE \left[ \frac{\fronorm{\Pihat A \Xhat - \Pi^* A
        X^*}^2}{nm}\right] & \geq c_2 \sigma^2\left(
  \frac{\rank(A)}{n} + \frac{1}{m} \right), \label{eq:minimaxlb}
\end{align}
where the constant $c_2$ depends on the value of the pair $(C_1,
C_2)$, but is independent of other problem parameters.
\end{subequations}
\end{theorem}

Theorem~\ref{thm:MLrate} characterizes the minimax rate up to a factor
that is at most logarithmic in $n$. It shows that the MLE is minimax
optimal for prediction error up to logarithmic factors for all
matrices that are not too flat.  The bounds have the following
interpretation, similar to the results of Flammarion et
al.~\cite{rigollet} on prediction error for unimodal columns. The
first term corresponds to a rate achieved even if the estimator knows
the true permutation $\Pi^*$; the second term quantifies the price
paid for the combinatorial choice among $n!$ permutations. As a
result, we see that if $m \gg \log n$, then the permutation does not
play much of a role in the problem, and the rates resemble those of
standard linear regression. Such a general behaviour is expected,
since a large $m$ means that we get multiple observations with the
same unknown permutation, and this should allow us to estimate
$\Pihat$ better.

Clearly, a flat matrix is not influenced by the unknown permutation,
and so the second term of the lower bound need not apply. As we
demonstrate in the proof, it is likely that the flatness of $A$ can
also be incorporated in order to prove a tighter upper bound in this
case, but we choose to state the upper bound as holding uniformly for
all matrices $A$, with the loss of a logarithmic factor.

It is also worth mentioning that the logarithmic factor in the second
term is shown to be nearly tight for the problem of unimodal matrix
estimation with an unknown permutation \cite{rigollet}, suggesting
that a similar factor may also appear in a tight version of our lower
bound~\eqref{eq:minimaxlb}. For the specific case where $m=1$ however,
which corresponds to the shuffled vector model~\eqref{shuffvec}, our
bounds are tight up to constant factors, and summarized by the
following corollary.

\begin{corollary}
\label{cor:MLrate1}
In the case $m =1$, for any matrix $A \in \mathcal{A}(C_1,
C_2/\sqrt{n})$, we have
\begin{align}
c_2 \sigma^2 \leq \inf_{\substack{\Pihat \in \mathcal{P}_n \\ \xhat
    \in \real^{d}}} \; \;\sup_{\substack{\Pi^* \in \mathcal{P}_n
    \\ x^* \in \real^{d}}} \EE \left[ \frac{1}{n} \| \Pihat A \xhat -
  \Pi^* A x^* \|_2^2 \right] \leq c_1 \sigma^2. \nonumber
\end{align}
\end{corollary}
\noindent In other words, the normalized minimax prediction error for
the shuffled vector model does not decay with the parameters $n$ or $d$,
 and so no estimator achieves consistent prediction for every
parameter choice $(\Pi^*, x^*)$. Again, this is a consequence of the
fact that---unlike when $m$ is large---we do not get independent
observations with the permutation staying fixed, and herein lies the
difficulty of the problem.

Both Theorem~\ref{thm:MLrate} and Corollary~\ref{cor:MLrate1} provide
non-adaptive minimax bounds. An interesting question is whether the
least squares estimator is also minimax optimal up to logarithmic
factors over finer classes of $\Pi^*$ and $X^*$, i.e., whether it is
adaptive in some interesting way. One would expect that the estimator
adapts to the parameter $\kappa(AX^*)$, the number of distinct entries
in the matrix $A X^*$, similarly to the problem of monotone parameter
recovery~\cite{rigollet}.


\subsection{Polynomial time estimators}
\label{sec:polytime}

As shown in our past work~\cite{permpaper}, computing the MLE
estimate~\eqref{eq:ml} is NP-hard in general.  Accordingly, it is
natural to turn our attention to alternative estimators, and in
particular ones that are guaranteed to run in polynomial time.

Here we analyze two simple methods for estimating the matrix $\Pi^* A
X^*$, based either on singular value thresholding, and a closely
related variant that uses an explicit regularization based on the
nuclear norm. It is well-known that such methods are appropriate when
the matrix is low-rank, or approximately low-rank. While the matrix
$Y^*$ is not low-rank, its rank is bounded by that of the matrix $A$,
a fact that we leverage in our bounds.

Given a matrix $M$ with the singular value decomposition $M = \sum_{i
  = 1}^r \sigma_i u_i v_i^\top$, its singular value thresholded
version at level $\lambda$ is given by $T_\lambda (M) = \sum_{i = 1}^r
\sigma_i \mathbb{I}(\sigma_i \geq \lambda) u_i v_i^\top$, where
$\mathbb{I}(\cdot)$ is the indicator function of its argument.

The singular value thresholding (SVT) operation serves the purpose of
denoising the observation matrix, and has been analyzed in the context
of more general matrix estimation problems by various authors
(e.g.,~\cite{cai2010singular,chatterjee}).

\begin{theorem}
\label{thm:svt}
For any matrices $(\Pi^*, X^*)$, the SVT estimate with \mbox{$\lambda
  = 1.1\sigma(\sqrt{n} +\sqrt{m})$} satisfies
\begin{subequations}
\begin{align}
\frac{1}{nm} \frobnorm{T_{\lambda} (Y) - \Pi^* A X^* }^2 \leq c_1 \sigma^2
\rank(A) \left( \frac{1}{n} + \frac{1}{m} \right)
\end{align}
with probability greater than $1 - e^{-cnm}$.

Conversely, for any matrix $A$ with rank at most $m$, there exist
matrices $\Pi_0$ and $X_0$ (that may depend $A$) such that
for any threshold $\lambda > 0$, we have
\begin{align}
\label{eq:svtlb}
\frac{1}{nm} \frobnorm{ T_{\lambda} (Y) - \Pi_0 A X_0 }^2 \geq c_2
\sigma^2 \rank(A) \left( \frac{1}{n} + \frac{1}{m} \right),
\end{align}
with probability greater than $1 - e^{-cnm}$.
\end{subequations}
\end{theorem}

Comparing inequalities~\eqref{eq:minimaxlb} (which holds for any
denoised matrix, not just those having the form $\Pihat A \Xhat$)
and~\eqref{eq:svtlb}, we see that the SVT estimator, while
computationally efficient, may be statistically sub-optimal. However,
it is consistent in the case where $\rank(A)$ is sufficiently small
compared to $m$ and $n$, and minimax optimal when $\rank(A)$ is a
constant. Intuitively, the rate it attains is a result of treating the
full matrix $\Pi^* A$ as unknown, and so it is likely that better,
efficient estimators exist that take the knowledge of $A$ into
account.

A potential concern is that the SVT estimator is required to know the
noise variance $\sigma^2$. This issue can be taken care of via the
square-root LASSO ``trick''~\cite{belloni2011square}, which ensures a
self-normalization that obviates the necessity for a noise-dependent
threshold level. In particular, consider the estimate
\begin{align}
   \label{eq:srlasso}
\widehat{Y}_{{\sf sr}}(\lambda) = \arg \min_{Y'} \frobnorm{Y - Y'} +
\lambda \nucnorm{Y'}.
\end{align}
Using a choice of $\lambda$ that no longer depends on $\sigma$, we
have the following guarantee:
\begin{theorem}
  \label{thm:srlasso}
If $\rank(A) \left( \frac{1}{n} + \frac{1}{m} \right) \leq 1/20$, then
for any choice of parameters $\Pi^*$ and $X^*$, the square-root LASSO
estimate \eqref{eq:srlasso} with $\lambda = 2.1 \left(
\frac{1}{\sqrt{n}} + \frac{1}{\sqrt{m}} \right)$ satisfies
\begin{align}
\frac{1}{nm} \frobnorm {\Yhat_{\mathsf{sr}}(\lambda) - \Pi^* A X^* }^2
\leq c_1 \sigma^2 \rank(A) \left( \frac{1}{n} + \frac{1}{m} \right)
\nonumber
\end{align}
with probability greater than $1 - 2e^{-cnm}$.
\end{theorem}
We prove Theorem~\ref{thm:srlasso} in Section \ref{sec:thm3proof} for
completeness.  However, it should be noted that the square-root LASSO
has been analyzed for matrix completion
problems~\cite{klopp2014noisy}, and our proof follows similar lines
for our different observation model.  The condition $\rank(A) \left(
\frac{1}{n} + \frac{1}{m} \right) \leq 1/20$ does not significantly
affect the claim, since our bounds no longer guarantee consistency of
the estimate $\Yhat_{\mathsf{sr}}(\lambda)$ when this condition is
violated.

While the optimization problem~\eqref{eq:srlasso} can be solved
efficiently, there may be cases when the noise is (sub)-Gaussian of
known variance for which the SVT estimate can be computed more
quickly. Hence, the SVT estimator is usually preferred in cases where
the noise statistics are known.


\subsection{Exact algorithm for the noiseless case}

For the noiseless model, the only efficient algorithm known up to now
is for the special case $d = m = 1$, as presented in our past
work~\cite{permpaper}. It turns out that this algorithm has a natural
generalization to higher dimensional problems, at least when certain
conditions on the input matrices $(A, Y)$ are satisfied.
The higher dimensional generalization requires analyzing certain spectral
properties of the input matrices.

In order to state the theorem, we require require a few definitions.
Given a matrix \mbox{$M \in \real^{\numobs \times \usedim}$,} consider
its reduced singular value decomposition \mbox{$M = U_M \Sigma_M
  V_M^\top$,} where $U_M$ is a matrix of its left singular vectors.
The (left) \emph{leverage scores} of the matrix $M$ are given the
$\ell_2$-norms of the rows of the matrix $U_M$; in analytical terms,
we can express them as the $\numobs$-dimensional vector $\ell(M) =
\diag(U_M U_M^\top)$, where the operator $\diag$ extracts the diagonal
of a square matrix.  With this notation, the {\sf LevSort} algorithm
performs the following three steps on the input pair $(Y, A)$:
\begin{enumerate}
\item[(i)]  Compute the leverage scores $\ell(Y)$ and $\ell(A)$.
\item[(ii)] Find a permutation $\PiLev \in \arg\min_\Pi \| \ell(Y) - \PiLev \;
  \ell(A) \|_2^2$.
\item[(iii)] Return the matrix $\XLev = \big( \PiLev A \big)^\dagger
  Y$, where $M^\dagger$ denotes the Moore-Penrose pseudoinverse of a
  matrix $M$.
\end{enumerate}
Note that this algorithm runs in polynomial time, since it involves
only spectral computations and a matching step that can be computed in time
 $O(n \log n)$. As we demonstrate in the proof, step (ii) for the noiseless
model actually returns a permutation matrix $\PiLev$ such that 
$\ell(Y) = \PiLev \ell(A)$.

\begin{theorem} \label{thm:algo}
Consider an instantiation of the noiseless model with $\rank(A) \leq
\rank(X^*)$, and such $\ell(A)$ and $\ell(Y)$ both have all distinct entries.
Then the {\sf LevSort} algorithm recovers the
parameters $(\Pi^*, X^*)$ exactly.
\end{theorem}

The {\sf LevSort} algorithm is a generalization of our own
algorithm~\cite{permpaper} to the matrix setting. However, instead of
a simple sorting algorithm, we now require an additional spectral
component. While showing the necessity of the condition $\rank(A) \leq
\rank(X^*)$ is still open, an efficient algorithm that does not impose
any conditions is unlikely to exist due to the general problem being
NP-hard~\cite{permpaper}. Note that the condition includes as a
special case all problems in which the matrices $A$ and $X^*$ are full
rank, with $d \leq m$.

In particular, the pose and correspondence estimation problem for 2D
point clouds satisfies the conditions of Theorem~\ref{thm:algo} under
some natural assumptions. We have $d = m = 2$ for all such problems,
and $\rank(X^*) = 2$ unless the linear transformation is
degenerate. Furthermore, unless the keypoints are generated
adversarially, the leverage scores of the matrix $A$ and the rows of
$Y$ are distinct. Thus, assuming that the noiseless version of
model~\eqref{obsmodel} exactly describes the keypoints detected in the
two images (which is an idealization that may not be true in real
data), we are guaranteed to find both the pose and the correspondence
exactly.

In Figure~\ref{fig:algo}, we demonstrate the guarantee of
Theorem~\ref{thm:algo} on two image correspondence tasks when the
keypoints detected in the two images are identical and the
transformation between coordinates is linear.
\begin{figure}
  \begin{center}
    \begin{tabular}{cc}
      \widgraph{0.45\textwidth}{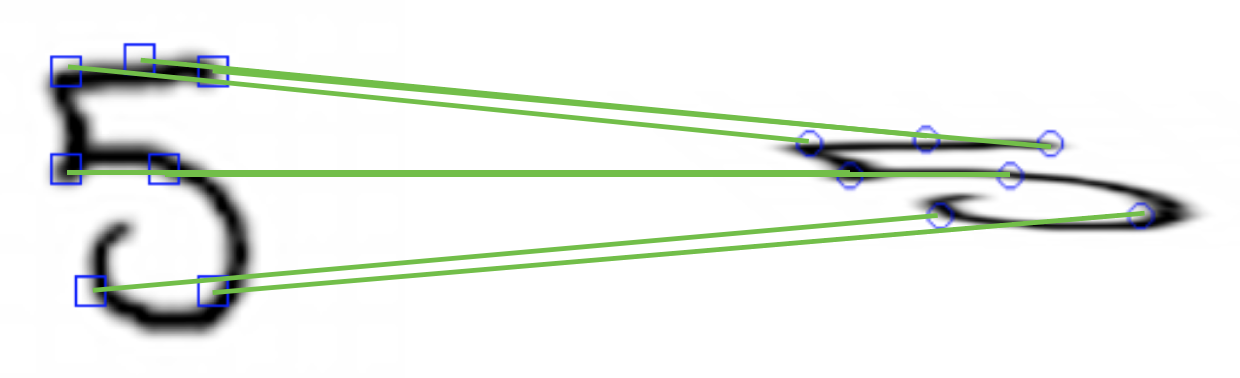} & 
      \widgraph{0.40\textwidth}{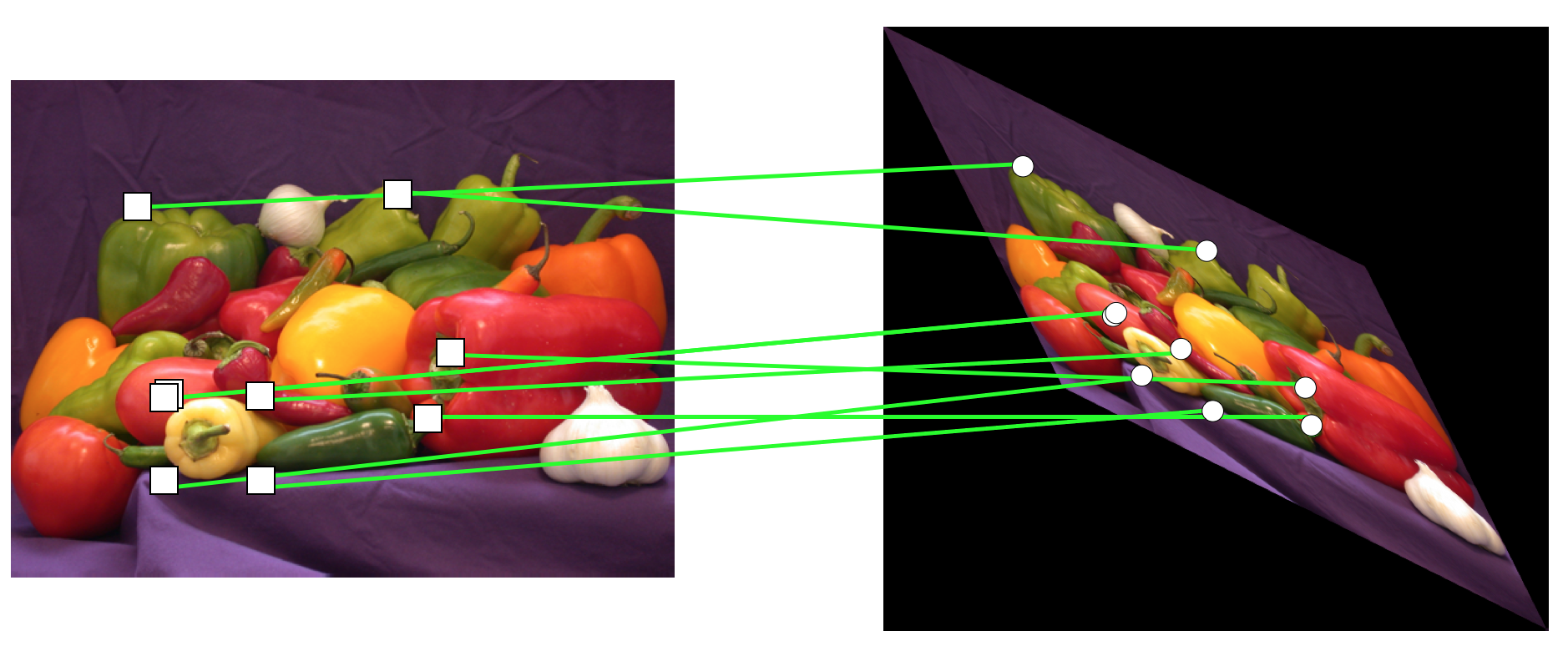} \\
      (a) & (b)
    \end{tabular}
    \caption{Synthetic experiment illustrating exact pose and
      correspondence estimation by the {\sf LevSort} algorithm for a
      transformed ``5'' in panel (a), and a transformed fruit picture
      in panel (b).  In each panel, the right images are obtained via
      a linear tranformation of the coordinates of the respective left
      images, and keypoints are generated according to the noiseless
      model~\eqref{obsmodel}; keypoints are the same in the right and
      left image. \vspace{-8mm}} \label{fig:algo}
    \end{center}
\end{figure}


\subsection{Extensions to outliers}

The results of Sections~\ref{sec:minimax} and~\ref{sec:polytime} also
hold in a somewhat general setting, where the set of perturbations to
the rows of the matrix $A$ is allowed to be larger than just the set
of permutation matrices $\PERMN$. In particular, defining the set of
``clustering matrices" $\mathcal{C}_n$ as
\begin{align}
\mathcal{C}_n = \{D \in \{0,1\}^{n \times n}\; \, \mid \, D {\bf 1} =
        {\bf 1}\}, \nonumber
\end{align}
we consider an observation model of the form
\begin{align}
Y = D^* A X^* + W, \label{cluster}
\end{align}
where the matrices $A$, $X^*$, and $W$ are as before, and $D^* \in
\mathcal{C}_n$ now represents a clustering matrix. Such a clustering
condition ensures stochasticity of the matrix $D^*$ (not double
stochasticity, as in the permutation model), and corresponds to the
case where multiple responses may come from the same covariate, and
some of the data may be permuted. Such a model is likely to better fit
data from image correspondence problems when the keypoints detected in
the two images are quite different. Also, such a formulation is
loosely related to the $k$-means clustering problem with Gaussian
data~\cite{awasthi2015relax}.

As it turns out, Theorems \ref{thm:MLrate}, \ref{thm:svt} and
\ref{thm:srlasso} also hold for this model, with minor modifications
to the proofs. Defining the analogous MLE for this model as
\begin{align}
\left(\widehat{D}_{\mathsf{ML}}, \Xhatml \right) = \arg
\min_{\substack{D \in \mathcal{C}_n \\ X \in \real^{d \times m}}}
\frobnorm{ Y - D A X }^2, \nonumber
\end{align}
we have the following theorem.

\begin{theorem} \label{thm:clustering}
\begin{enumerate}
\item[(a)] For any matrix $A$, and for all parameters $X^* \in
  \real^{d \times m}$ and $D^* \in \mathcal{C}_n$, we have
\begin{align}
\frac{ \frobnorm{ \widehat{D}_{\mathsf{ML}} A \Xhatml - D^* A X^*
  }^2}{nm} & \leq c_1 \sigma^2 \left( \frac{\rank(A)}{n} + \frac{1}{m}
\min\left\{\log n, m \right\} \right), \nonumber
\end{align}
with probability greater than $1 - e^{-c(n\log n + m \rank(A))}$.
\item[(b)] For any choice of parameters $D^*$ and $X^*$, the SVT
  estimate with $\lambda = 1.1 \sigma (\sqrt{n} + \sqrt{m})$ satisfies
\begin{align}
\frac{1}{nm} \frobnorm{ T_{\lambda} (Y) - D^* A X^* }^2 \leq c_1
\sigma^2 \rank(A) \left( \frac{1}{n} + \frac{1}{m} \right) \nonumber
\end{align}
with probability greater than $1 - e^{-cnm}$.
\item[(c)] For any choice of parameters $D^*$ and $X^*$, the
  square-root LASSO estimate~\eqref{eq:srlasso} with $\lambda = 2.1
  \left( \frac{1}{\sqrt{n}} + \frac{1}{\sqrt{m}} \right)$ satisfies
\begin{align}
\frac{1}{n m} \frobnorm{ \Yhat_{\mathsf{sr}}(\lambda) - D^* A X^* }^2
\leq c_1 \sigma^2 \rank(A) \left( \frac{1}{n} + \frac{1}{m} \right)
\nonumber
\end{align}
\end{enumerate}
with probability greater than $1 - 2e^{-cnm}$.
\end{theorem}

Clearly, the lower bounds~\eqref{eq:minimaxlb} and~\eqref{eq:svtlb}
hold immediately for the model~\eqref{cluster} as a result of the
inclusion $\mathcal{P}_n \subset \mathcal{C}_n$.


\section{Proofs}
\label{sec:proofs}

This section contains proofs of all our main results. We use $C, c,
c'$ to denote absolute constants that may change from line to line. We
let $\sigma_i(M)$ denote the $i$th largest singular value of a matrix
$M$.


\subsection{Proof of Theorem~\ref{thm:MLrate}}

We split the proof into two natural parts, corresponding to the upper
and lower bounds, respectively.  The upper bound boils down to
analyzing the Gaussian width~\cite{pisier1999volume} of a certain set,
which we obtain via Dudley's entropy integral~\cite{dudley} and bounds
on the metric entropy of the observation space. The lower bound is
obtained via a packing construction and an application of Fano's
inequality.

\subsubsection{Proof of upper bound}
\label{sec:thm1ub}

Writing $Y^* = \Pi^* A X^*$ and $\widehat{Y} = \Pihatml A
\widehat{X}_{{\sf ML}}$, we have by the optimality of $\Yhat$ for
problem~\eqref{eq:ml} that $\frobnorm{Y - \Yhat }^2 \leq \frobnorm{Y -
  Y^* }^2$, from which it follows that the error matrix $\Deltahat =
\Yhat - Y^*$ satisfies the following \emph{basic inequality}:
\begin{align}
  \label{basic}
\frac{1}{2} \frobnorm{\widehat{\Delta}}^2 & \leq 
\tracer{\widehat{\Delta}}{W},
\end{align}
where $\tracer{A}{B}$ denotes the trace inner product between two
matrices $A$ and $B$. We prove inequality~\eqref{upperbound} by
proving the following claims.

\begin{subequations}
\begin{align}
&\Pr\left\{ \frac{\frobnorm{\Deltahat}^2}{nm} \geq 8 \sigma^2 \right\}
  \leq e^{-\frac{nm}{8}}, \text{ and} \label{in1} \\
& \Pr\left\{ \frac{\frobnorm{\Deltahat}^2}{nm} \geq c_2 \sigma^2
  \left( \frac{d}{n} + \frac{\log n}{m} \right) \right\} \leq
  e^{-c(n\log n + m \rank(A))}. \label{in2}
\end{align}
\end{subequations}

\paragraph{Proof of inequality~\eqref{in1}:}
Applying the Cauchy Schwarz inequality to the RHS of
inequality~\eqref{basic} yields
\begin{align}
  \label{csbound}
\frac{1}{2} \frobnorm{\Deltahat} & \leq \frobnorm{W}.
\end{align} 

Squaring both sides of inequality \eqref{csbound} and using standard
sub-exponential tail bounds~\cite{wainwrig2015high} yields inequality
\eqref{in1}.  \qed

\paragraph{Proof of inequality~\eqref{in2}:}

Without loss of generality, by rescaling as necessary, we may assume
that the noise $W$ has standard normal entries ($\sigma^2 = 1$). We
use $\mathbb{U}_m(A)$ to denote the set of matrices whose $m$ columns
lie in the range of $\Pi A$ for some permutation matrix $\Pi$, i.e.,
\begin{align}
\mathbb{U}_m(A) & = \{Y \in \real^{n \times m} \, \mid \, Y = \Pi A X
\text{ for some } \Pi \in \mathcal{P}_n, X \in \real^{d \times m}
\}. 
\end{align}
Also define the set
\begin{align*}
\mathbb{U}^{{\sf diff}}_m(A) & = \{Y \, \mid \, Y = Y_1 - Y_2 \text{
  for } Y_1, Y_2 \in \mathbb{U}_m(A)\},
\end{align*}
as well as the function
\begin{align*}
Z(t) & \defn \sup_{\substack{D \in \mathbb{U}^{{\sf diff}}_m(A)
    \\ \frobnorm{D} \leq t}} \tracer{D}{W}.
\end{align*}
Before proceeding with the proof, we state the definition of the
covering number of a set.
\begin{definition}[Covering number]
A $\delta$-cover of a set $\mathbb{T}$ with respect to a metric $\rho$
is a set $\left\{ \theta^1, \theta^2, \ldots, \theta^N\right\} \subset
\mathbb{T}$ such that for each $\theta \in \mathbb{T}$, there exists
some $i \in [N]$ such that $\rho(\theta, \theta_i) \leq \delta$. The
$\delta$-covering number $N(\delta, \mathbb{T}, \rho)$ is the
cardinality of the smallest $\delta$-cover.
\end{definition}
The logarithm of the covering number is referred to as the metric
entropy of a set.  The following lemma bounds the metric entropy of
the set $\udiff$. Let $\mathbb{B}_F(t)$ denote the Frobenius norm ball
of radius $t$ centered at $0$.
\begin{lemma} \label{lem:metent}
The metric entropy of the set $\udiff \cap \mathbb{B}_F(t)$ in the
Frobenius norm metric is bounded as
\begin{align}
\log N(\delta, \udiff \cap \mathbb{B}_F(t), \frobnorm{\cdot} ) & \leq
2 \rank(A) \cdot m \log \left( 1 + \frac{4t}{\delta} \right) + 2 n\log
n.
\end{align}
\end{lemma}
\noindent We prove the lemma at the end of the section, taking it as
given for the proof of inequality~\eqref{in2}.

\begin{proof}[Proof of inequality~\eqref{in2}]
By definition of $Z(t)$, is easy to see that we have
\begin{align*}
\frac{1}{2} \frobnorm{\widehat{\Delta}}^2 \leq Z\big(
\frobnorm{\Deltahat} \big).
\end{align*}
One can also verify that the set $\udiff$ is star-shaped\footnote{A
  set $S$ is said to be star-shaped if $t \in S$ implies that $\alpha
  t \in S \text{ for all } \alpha \in [0,1]$}, and so the following
critical inequality holds for some $\deltanm > 0$:
\begin{align}
\EE \left[ Z(\deltanm) \right] \leq \frac{\deltanm^2}{2}. \label{crit}
\end{align}
We are interested in the smallest (strictly) positive solution to
inequality \eqref{crit}. Moreover, we would like to show that for
every $t \geq \deltanm$, we have $\frobnorm{\Deltahat}\leq c \sqrt{t
  \deltanm}$ with probability greater than $1 - c e^{-c' t \deltanm}$.

Define the ``bad" event
\begin{align}
\mathcal{E}_t & \defn \Big\{ \exists \Delta \in \udiff \mid
\frobnorm{\Delta} \geq \sqrt{t \deltanm} \text{ and }
\tracer{\Delta}{W} \geq 2 \frobnorm{\Delta} \sqrt{t \deltanm} \Big\}.
\end{align}
Using the star-shaped property of $\udiff$, it follows by a rescaling
argument that
\begin{align*}
\Pr[ \mathcal{E}_t ] \leq \Pr[ Z(\deltanm) \geq 2\deltanm \sqrt{t
    \deltanm} ] \text{ for all } t \geq \deltanm.
\end{align*}
The entries of $W$ are i.i.d. standard Gaussian, and the function $W \mapsto Z(t)$ is
convex and Lipschitz with parameter $t$. Consequently, by Borell's
theorem~(see, for example, Milman and Schechtman \cite{milman1986asymptotic} for a simple proof), the following holds for all
$t \geq \deltanm$:
\begin{align*}
\Pr [ Z(\deltanm) \geq \EE [Z(\deltanm)] + \deltanm \sqrt{t\deltanm} ]
\leq 2 e^{-c t\deltanm}. 
\end{align*}
By the definition of $\deltanm$, we have $\EE [Z(\deltanm )] \leq
\deltanm^2 \leq \deltanm \sqrt{t \deltanm}$ for any $t \geq \deltanm$,
and consequently, for all $t \geq \deltanm$, we have
\begin{align*}
\Pr[\mathcal{E}_t] \leq \Pr[Z(\deltanm) \geq 2\deltanm \sqrt{
    t\deltanm} ] \leq 2e^{-ct\deltanm}.
\end{align*}
Now either $\frobnorm{\Delta} \leq \sqrt{t\deltanm}$, or we have
$\|\Delta\|_F > \sqrt{t\deltanm}$. In the latter case, conditioning on
the complementary event $\mathcal{E}^c_t$, our basic inequality
implies that $\frac{1}{2} \frobnorm{\Delta}^2 \leq 2 \frobnorm{\Delta}
\sqrt{t \deltanm}$. Consequently, we have
\begin{align*}
\Pr\left\{ \frobnorm{\Delta} > 4\sqrt{t\deltanm}\right\} & \leq
\Pr\left\{ \frobnorm{\Delta} > 4\sqrt{t\deltanm} | \mathcal{E}^c_t
\right\} + \Pr\{ \mathcal{E}_t\} \\
& \leq 2e^{-ct\deltanm}. 
\end{align*}
Putting together the pieces yields
\begin{align}
\frobnorm{\Delta} \leq c_0 \sqrt{t \deltanm} \nonumber
\end{align}
with probability at least $1 - 2e^{-ct\deltanm}$ for every $t \geq
\deltanm$.

In order to determine a feasible $\deltanm$ satisfying the critical
inequality~\eqref{crit}, we need to bound the expectation
$\EE[Z(\deltanm)]$.  We now use Dudley's entropy integral ~\cite{dudley} to bound
$\EE[Z(t)]$. In particular, for a universal constant $C$, we have
\begin{align*}
\frac{1}{C}\EE\left[ Z(t) \right] &\leq \int_0^t \sqrt{ \log N(\delta,
  \udiff \cap \mathbb{B}_F(t), \frobnorm{\cdot} )} d\delta \\
& \stackrel{\1}{\leq} t \sqrt{n \log n} + \sqrt{m\rank(A)} \int_0^t
\sqrt{\log \left( 1+ \frac{2t}{\delta} \right) }d\delta \\
& \stackrel{\2}{=} t \sqrt{n \log n} + t \sqrt{m\rank(A)}
\int_0^1 \sqrt{\log \left( 1+ \frac{2}{u} \right) }du \\
& \leq t \sqrt{n \log n} + c t \sqrt{m \rank(A)},
\end{align*}
where in step $\1$, we have made use of Lemma~\ref{lem:metent}, and in
step $\2$, we have used the change of variables $u = \delta/t$. Now
comparing with the critical inequality, we see that $$\delta_n \leq c
\left(\sqrt{n\log n} + \sqrt{m \rank(A)}\right).$$ Putting together
the pieces then proves claim~\eqref{in2}.
\end{proof}

\noindent It remains to prove Lemma~\ref{lem:metent}.

\paragraph{Proof of Lemma~\ref{lem:metent}.}
We begin by finding the $\delta$-covering number of
\begin{align}
\mathbb{U}_m^\Pi(A) = \{ Y \in \real^{n \times m} \mid Y = \Pi A X
\text{ for some } X \in \real^{d \times m}\}.
\end{align}
Note that $\mathbb{U}_m^\Pi$ is isomorphic to $\range(I_m \otimes \Pi
A)$, where $\otimes$ denotes the tensor product. Note that $\range(I_m
\otimes \Pi A)$ is a linear subspace of dimension $m \cdot
\rank(A)$. Also, since the set \mbox{$\range(I_m \otimes \Pi A) \cap
  \mathbb{B}_2^{nm} (t)$} is an $m \cdot \rank(A)$-dimensional
$\ell_2$-ball of radius $t$, we have by a volume ratio argument that
\begin{align*}
  N (\delta, \mathbb{U}_m^\Pi(A) \cap \mathbb{B}_F(t),
  \frobnorm{\cdot} ) \leq \left(1 + \frac{2t}{\delta} \right)^{m
    \rank(A)}.
\end{align*}

By definition, we also have $\mathbb{U}_m (A) = \bigcup_{\Pi \in
  \mathcal{P}_n}\mathbb{U}_m^\Pi(A)$, and so by the union bound, we
have
\begin{align*}
N(\delta, \mathbb{U}_m (A) \cap \mathbb{B}_F(t), \frobnorm{\cdot} )
\leq n! \left(1 + \frac{2t}{\delta} \right)^{m \rank(A)}.
\end{align*}
In order to complete the proof, we notice that
\begin{align*}
N(\delta, \udiff \cap \mathbb{B}_F(t), \frobnorm{\cdot} ) &\leq \left[
  N(\delta/2, \mathbb{U}_m (A) \cap \mathbb{B}_F(t), \frobnorm{\cdot})
  \right]^2,
\end{align*}
since it is sufficient to use two $\delta/2$-covers of the set
$\mathbb{U}_m (A) \cap \mathbb{B}_F(t)$ in conjunction in order to
obtain a $\delta$-cover of the set $\udiff \cap \mathbb{B}_F(t)$.
\qed

\subsubsection{Proof of lower bound}

As alluded to before, the bound follows from a packing set
construction and Fano's inequality, which is a standard template used
to prove minimax lower bounds. Suppose we wish to estimate a parameter
$\theta$ over an indexed class of distributions $\mathcal{P} =
\{\mathbb{P}_\theta \; \mid\; \theta \in \Theta \}$ in the square of a
(pseudo-)metric $\rho$. We refer to a subset of parameters
$\{\theta^1, \theta^2, \ldots, \theta^M \}$ as a local $(\delta,
\epsilon)$-packing set if
\begin{align*}
\min_{i,j \in [M], i\neq j} \rho(\theta^i, \theta^j) \geq \delta
\qquad \text{ and } \qquad \frac{1}{\binom{M}{2}} \sum_{i, j \in [M]}
D(\mathbb{P}_{\theta^i} \| \mathbb{P}_{\theta^j}) \leq \epsilon.
\end{align*}
Note that this set is a $\delta$-packing in the $\rho$ metric with the
average KL-divergence bounded by $\epsilon$.  The following result is
a straightforward consequence of Fano's inequality:
\begin{lemma}[Local packing Fano lower bound]
  \label{fanolbprime}
For any $(\delta, \epsilon)$-packing set of cardinality $M$, we have
\begin{align}
\inf_{\thetahat} \sup_{\thetastar \in \Theta} \EE \left[
  \rho(\thetahat, \thetastar)^2\right] \geq \frac{\delta^2}{2} \left(1
- \frac{\epsilon + \log 2}{\log M} \right).
\end{align}
\end{lemma}

The remainder of argument is directed to establishing the following
two claims:
\begin{subequations}
\begin{align}
\sup_{X^* \in \real^{d \times m}} \EE \frac{\frobnorm{\Pihat A \Xhat -
    \Pi^* A X^*}^2}{nm} \geq c \sigma^2 \frac{\rank(A)}{n} \text{ for
  all } A, \text{ and}\label{lb1prime} \\
\sup_{X^* \in \real^{d \times m}} \EE \frac{\frobnorm{\Pihat A \Xhat -
    \Pi^* A X^*}^2}{nm} \geq c' \sigma^2 \frac{1}{m} \text{ if } A \in
\mathcal{A}(C_1, C_2/\sqrt{n}). \label{lb2prime}
\end{align}
\end{subequations}
It is easy to see that both claims together prove the lemma.


\paragraph{Proof of claim~\eqref{lb1prime}:}

This claim is consequence of classical minimax bounds on linear
regression. Since we are operating in the matrix setting, we include
the proof for completeness.

The proof involves the construction of a packing set $\{\Pi A X_i
\}_{i=1}^M$ such that for all \hbox{$i \neq j \in [M]$}, we have
$\frac{\frobnorm{\Pi A X_i}}{\sqrt{nm}} \leq 4\delta$ and
$\frac{\frobnorm{\Pi A X_i - \Pi A X_j}}{\sqrt{nm}} \geq
\delta$. Since we are effectively packing the space
$\frac{1}{\sqrt{nm}}\range(I_m \otimes \Pi A)$, standard results show
that there exists such a packing of this space with $\log M \geq
\rank(I_m \otimes \Pi A) \log 2$.

Also note that with the underlying parameter $X_i$, our observations
have the distribution $\mathbb{P}_i = \NORMAL(\Pi A X_i, \sigma^2
I_{nm})$. Hence, the KL divergence between two observations $i$ and
$j$ is simply
\begin{align*}
D(\mathbb{P}_i \| \mathbb{P}_j) = \frac{1}{\sigma^2} \| \Pi A X_i -
\Pi A X_j \|_F^2 \leq \frac{32 \delta^2 nm}{\sigma^2}.
\end{align*}
Substituting this into the bound of Lemma~\ref{fanolbprime} with
$\rho(\theta_1, \theta_2) = \| \theta_1 - \theta_2 \|_F$, we have
\begin{align*}
\mathcal{M} \geq \frac{\delta^2}{2} \left(1 - \frac{ \frac{32 \delta^2
    nm}{\sigma^2} + \log 2}{m \rank(A) \log 2} \right),
\end{align*}
where we have again used $\mathcal{M}$ to denote the minimax rate of
prediction.

Setting $\delta^2 = c\frac{\sigma^2 \rank(A)}{n}$ completes the proof
of claim~\eqref{lb1prime}. Note that the proof of this claim did not
require the assumption that $A \in \mathcal{A}(C_1, C_2/ \sqrt{n})$.


\paragraph{Proof of claim~\eqref{lb2prime}}

For ease of exposition, we first prove claim~\eqref{lb2prime} for matrices in a smaller class
than $\mathcal{A}(C_1, C_2/ \sqrt{n})$.  We let $\mathbf{1}^{p}_n$ denote the $n$-dimensional
vector having $1$ in its first $p$ coordinates and $0$ in the
remaining coordinates.

Now consider the class of matrices that have $\mathbf{1}^{p}_n$ in
their range. By multiplying with $\delta$ and stacking $m$ of these
vectors up as columns, we have a matrix $\widetilde{Y}^1 \in \real^{n
  \times m}$ whose first $p$ rows are identically $\delta$ and the
rest are identically zero. Define the Hamming distance between two
binary vectors $\dH(u, v) = \#\{i: u_i \neq v_i\}.$ We require the
following lemma.
\begin{lemma}
  \label{lem:packingprime}
There exists a set of binary $n$-vectors $\{v_i\}_{i=1}^M$, each of
Hamming weight $p$ and satisfying $\dH(v_i, v_j) \geq h$, having
cardinality $M = \frac{\binom{n}{p}}{\sum_{i = 1}^{\lfloor
    \frac{h-1}{2} \rfloor} \binom{n - p}{i} \binom{p}{i}}.$
\end{lemma}
\noindent The lemma is proved at the end of this section.

\paragraph{Proof of claim~\eqref{lb2prime}}

Applying Lemma~\ref{lem:packingprime} and a rescaling argument, we see
that there is a packing set $\{ \Pi_i \widetilde{Y}^1 \}_{i=1}^M$ such
that
\begin{subequations}
\begin{align}
\frac{1}{\sqrt{nm}} \frobnorm{\Pi_i \widetilde{Y}^1} &= \delta
\sqrt{\frac{p}{n}}\text{ for } i \in [M], \text{ and } \\
\frac{1}{\sqrt{nm}} \frobnorm{\Pi_i \widetilde{Y}^1 - \Pi_j
  \widetilde{Y}^1} & \geq  \delta \sqrt{\frac{h}{n}} \text{ for } i
\neq j \in [M]. 
\end{align}
\end{subequations}
Fixing some constant $\gamma \in (0,1)$ and choosing $p = \gamma n$
and $h = \frac{n}{2}\min\left\{ \gamma/2, (1-\gamma)/2\right\}$, it
can be verified that we obtain a packing set of size $M \geq e^{\gamma
  \log(1/\gamma) n}$.  We now have observation $i$ distributed as
$\mathbb{P}_i = \NORMAL(\Pi_i \widetilde{Y}^1, \sigma^2 I_{nm})$, and
so
\begin{align*}
D(\mathbb{P}_i \| \mathbb{P}_j) = \frac{1}{\sigma^2} \frobnorm{ \Pi_i
  \widetilde{Y}^1 - \Pi_j \widetilde{Y}^1 }^2 \leq c \frac{\delta^2
  \gamma nm}{\sigma^2}.
\end{align*}

Finally, substituting into the Fano bound of Lemma~\ref{fanolbprime}
yields
\begin{align*}
\inf_{\substack{\Pihat \in \mathcal{P}_n \\ \Xhat \in \real^{d \times
      m}}} \sup_{\substack{\Pi^* \in \mathcal{P}_n \\ X^* \in \real^{d
      \times m}}} \EE \left[ \frac{1}{nm} \frobnorm{\Pihat A \Xhat -
    \Pi^* A X^*}^2 \right] \geq \frac{\delta^2}{2} \left(1 - \frac{
  \frac{c \delta^2 \gamma nm}{\sigma^2} + \log 2}{\gamma
  \log(1/\gamma) n} \right).
\end{align*}
Setting $\delta^2 = c(\gamma)\frac{\sigma^2}{m}$ for a constant
$c(\gamma)$ depending only on $\gamma$ completes the proof provided
the vector $\mathbf{1}^p_n \in { \sf range} (A)$ for $p = \gamma n$
with $\gamma \in (0,1)$.

It remains to extend the proof to matrices in the class
$\mathcal{A}(C_1, C_2/\sqrt{n})$, and to prove
Lemma~\ref{lem:packingprime}.

By definition, if $A \in \mathcal{A}(C_1, C_2/\sqrt{n})$, then there
exists a vector $a \in \range(A) \cap \mathbb{B}_2(1)$ such that
$a^s_{C_1 n} \geq a^s_{C_1 n+1} + C_2/\sqrt{n}$. We may assume that
$\|a\|_2 = 1$ by a rescaling argument, and also that $a = a^s$. By
definition, we have
\begin{align}
\left(a_i - a_j\right)^2 \geq C_2^2 / n \text{ for all } i \leq
\lfloor C_1 n \rfloor \text{ and } j \geq \lfloor C_1 n \rfloor +
1. \label{eq:sep}
\end{align}

It can also be verified that since $\|a\|_2 = 1$, we must have $C_2
\leq 2$. For the rest of the proof, we assume for simplicity of
exposition that $C_1 n$ is an integer. Fixing the value
\mbox{$\epsilon = \frac{n}{2} \min(C_1, 1 - C_1)$,} consider the
$\epsilon$-packing generated by permutations $\left\{ \Pi_i
\right\}_{i=1}^M$ of the vector $\mathbf{1}^{C_1 n}_n$, given by
Lemma~\ref{lem:packingprime} by taking $v_i = \Pi_i \mathbf{1}^{C_1
  n}_n$. Using these permutations, we observe that
\begin{align*}
\|\Pi_i a - \Pi_j a \|_2 &\geq \sqrt{\epsilon} \frac{C_2}{\sqrt{n}}
\geq c,
\end{align*}
where $c$ depends on the constants $(C_1, C_2)$, and we have used
condition \eqref{eq:sep} along with the fact that $\dH(v_i, v_j) \geq
\epsilon$.

Following similar steps to before then proves lemma for all matrices
$A \in \mathcal{A} (C_1, C_2/\sqrt{n})$.

\noindent It remains to prove Lemma~\ref{lem:packingprime}.

\paragraph{Proof of Lemma~\ref{lem:packingprime}}

The proof follows by a volume ratio argument that underlies the proof
of the Gilbert-Varshamov bound. In particular, the number of permuted
vectors of $\mathbf{1}_n^p$ that are within a Hamming distance $h -1$
of $\mathbf{1}_n^p$ is given by $\Delta = \sum_{i = 1}^{\lfloor
  \frac{h-1}{2} \rfloor} \binom{n - p}{i} \binom{p}{p - i}$. Now form
a graph with all $\binom{n}{p}$ permuted vectors of $\mathbf{1}_n^p$
as vertices and connect two vertices if the corresponding vectors have
Hamming distance less than $h$. Then such a graph has uniform degree
$\Delta$ and therefore contains an independent set of size
$\frac{\binom{n}{p}}{\Delta}$. \qed


\subsection{Proof of Theorem~\ref{thm:svt}}

Again, we divide our proof into two parts, corresponding to the upper
and lower bounds respectively.

\subsubsection{Proof of upper bound}
\label{sec:thm2ub}

For this proof, we use the shorthand $Y^* = \Pi^* A X^*$. Also fix
$\delta = 0.1$, and let $s$ be the number of singular values of $Y^*$
greater than $\frac{\delta}{1 + \delta} \lambda$. Also, let $Y^*_s$
denote the matrix formed by truncating $Y^*$ to its top $s$ singular
values. By triangle inequality, we have
\begin{align*}
\frobnorm{T_{\lambda} (Y) - Y^*}^2 &\leq 2 \frobnorm{T_{\lambda} (Y) -
  Y^*_s}^2 + 2 \frobnorm{Y^* - Y^*_s }^2 \\
& \leq 2 \rank(T_{\lambda} (Y) - Y^*_s) \opnorm{T_{\lambda} (Y) -
  Y^*_s}^2 + 2 \rank(Y^*) \left(\frac{\delta}{1 + \delta} \lambda
\right)^2.
\end{align*}
Now note that by standard results in random matrix theory (see, for
example,~\cite[Theorem 6.1]{wainwrig2015high}), we have $\lambda \geq
(1 +\delta) \opnorm{W}$ with probability greater than $1 -
e^{-\frac{\delta^2}{2} n(\sqrt{n} + \sqrt{m})^2}$. We condition on
this event for the rest of the proof.

Consequently, for $j \geq s + 1$, we have
\begin{align}
\sigma_j(Y) \leq \sigma_j(Y^*) + \opnorm{W} \leq \lambda, \nonumber
\end{align}
and so $\rank(T_{\lambda}(Y)) \leq s$. Additionally, we have
\begin{align*}
\opnorm{T_{\lambda} (Y) - Y^*_s} &\leq \opnorm{T_{\lambda} (Y) - Y} +
\opnorm{Y - Y^*} + \opnorm{Y^* - Y^*_s} \\
& \leq \lambda + \opnorm{W} + \frac{\delta}{1 + \delta} \lambda
\\
& \leq 2 \lambda.
\end{align*}
Putting together the pieces yields
\begin{align*}
\frobnorm{T_{\lambda} (Y) - Y^*}^2 & \leq 16 \lambda^2 s + 2
\rank(Y^*) \left(\frac{\delta}{1 + \delta} \lambda \right)^2\\
& \leq C \sigma^2 \rank(Y^*) (\sqrt{n} + \sqrt{m})^2, 
\end{align*}
a bound that holds with probability greater than $1 - e^{-c nm}$. In
order to complete the proof, we note that \mbox{$\rank(Y^*) \leq
  \rank(A)$.}  \qed


\subsubsection{Proof of lower bound}

We split our analysis into two separate cases.

\paragraph{Case 1:} First suppose that
$\lambda \leq \frac{\sigma}{3} (\sqrt{n} + \sqrt{m})$. Consider any
matrix $Y^* = \Pi^* A X^*$, and $Y = Y^* + W$. By definition of the
thresholding operation, we have
\begin{align*}
\frobnorm{T_\lambda (Y) - Y}^2 & \leq \min\{n, m\} \opnorm{T_\lambda
  (Y) - Y}^2 \leq \min\{n, m\} \lambda^2 \\
& \leq \frac{1}{9} \sigma^2 \min\{n, m\} (\sqrt{n} +
\sqrt{m})^2. 
\end{align*}
Triangle inequality yields
\begin{align*}
\frobnorm{T_\lambda (Y) - Y^*} & \geq \frobnorm{Y - Y^*} -
\frobnorm{T_\lambda (Y) - Y} \\
& \geq \frobnorm{ W } - \frac{1}{3} \sigma \sqrt{\min\{n, m\}}
(\sqrt{n} + \sqrt{m}).
\end{align*}
Now with probability greater than $1 - e^{-cnm}$, we have
$\frobnorm{W}^2 \geq \sigma^2 \frac{nm}{2}$, so that conditioned on
this event, we have
\begin{align*}
\frobnorm{T_\lambda (Y) - Y^*} & \geq \sigma \sqrt{nm}
\left(\frac{1}{\sqrt{2}} - \frac{2}{3}\right),
\end{align*}
which completes the proof.


\paragraph{Case 2:}
We now suppose that $\lambda > \frac{\sigma}{3} (\sqrt{n} +
\sqrt{m})$. Let the matrix $A$ have the (reduced) singular value
decomposition $A = U_A \Sigma_A V_A^\top$, and introduce the shorthand
$r \defn \rank(A)$. Form the diagonal matrix $L = \frac{\sqrt{n} +
  \sqrt{m}}{6} I_r$. Now let $\Pi_0 = I_n$, and consider the parameter
matrix $X_0 = V_A \Sigma_A^{-1} L V^{\top}$, where $V$ is an $m \times
\rank(A)$ dimensional matrix $V$ with orthonormal columns. Note that
such a choice exists when $\rank(A) \leq m$.

We now have
\begin{align*}
\frobnorm{T_\lambda (Y) - \Pi_0 A X_0}^2 &= \frobnorm{T_\lambda (U_A L
  V^\top + W) - U_A L V^\top}^2.
\end{align*}
For two matrices $A, B \in \real^{n \times m}$ with $k = \min\{n,
m\}$, it can be verified that
\begin{align*}
\frobnorm{A - B}^2 & \geq \sum_{i=1}^k \Big(\sigma_i(A) - \sigma_i (B)
\Big)^2.
\end{align*}
By the definition of the thresholding operation, the top singular
values of the matrix \mbox{$T_\lambda (U_A L V^\top + W)$} are all
either greater than $\lambda$, or equal to $0$. Hence, we have
\begin{align*}
\frobnorm{T_\lambda (Y) - \Pi_0 A X_0}^2 & \geq \sum_{i=1}^r \left(
\lambda \mathbb{I}\left\{\sigma_i(T_\lambda (U_A L V^\top + W)) \geq
\lambda \right\} - \frac{\sqrt{n} + \sqrt{m}}{6} \right)^2 \\
& \geq c r (n + m),
\end{align*}
where the last step follows since $\lambda > \frac{\sigma}{3}
(\sqrt{n} + \sqrt{m})$, which completes the proof.  \qed

\subsection{Proof of Theorem~\ref{thm:srlasso}}
\label{sec:thm3proof}

It is again helpful to write the observation model in the form $Y =
Y^* + W$, where $Y^* = \Pi^* A X^*$ represents the underlying matrix
we are trying to predict. Let us denote the choice of $\lambda$ in the
statement of Theorem~\ref{thm:srlasso} by \mbox{$\lambda_0 = 2.1 \:
  \frac{\sqrt{n} + \sqrt{m}}{\sqrt{nm}}$.} We use the shorthand $R(M)
= \frobnorm{Y - M}$, and $\Delta = Y^* - \Yhatsrl$. Let $P_M$ and
$P^\perp_M$ denote, respectively, the projection matrices onto the
rowspace of the matrix $M$ and its orthogonal complement.

We require the following auxiliary lemmas for our proof:
\begin{lemma}
  \label{lem:nuc}
We have
\begin{align*}
\nucnorm{Y^*} - \nucnorm{\Yhatsrl} & \leq \nucnorm{P_{Y^*} \Delta} -
\nucnorm{P^\perp_{Y^*} \Delta}. 
\end{align*}
\end{lemma}

\begin{lemma}
  \label{lem:lambda}
If $\lambda_0 \geq 2\frac{\opnorm{W}}{\frobnorm{W}}$, we have
\begin{align*}
\nucnorm{P^\perp_{Y^*} \Delta} & \leq 3 \nucnorm{P_{Y^*} \Delta}.
\end{align*}
\end{lemma}

We are now ready to prove Theorem~\ref{thm:srlasso}.


\begin{proof}[Proof of Theorem~\ref{thm:srlasso}]

First, note that by standard results on concentration of
$\chi^2$-random variables and random matrices (see, for instance,
Wainwright \cite{wainwrig2015high}), we have
\begin{subequations}
\begin{align}
\Pr\{ \opnorm{W} \geq 1.01 \sigma (\sqrt{n} + \sqrt{m}) \} &\leq
e^{-cnm}, \text{ and }\nonumber \\ \Pr\{ \frobnorm{W}\leq 0.99 \sigma
\sqrt{nm} \} &\leq e^{-c'nm}. \nonumber
\end{align}
\end{subequations}
Hence, we have
\begin{align*}
\Pr \left\{ \lambda_0 \geq 2 \frac{\opnorm{W}}{\frobnorm{W}} \right\}
\geq 1 - 2e^{cnm}.
\end{align*}
For the rest of the proof, we condition on the event $\{ \lambda_0
\geq 2 \frac{\opnorm{W}}{\frobnorm{W}}\}$.

Now, by definition of the quantity $R(M)$, we have
\begin{align*}
R(\Yhatsrl)^2 - R(Y^*)^2 &= \tracer{Y^* - \Yhatsrl}{Y^* - \Yhatsrl + 2
  W}\\
& = \frobnorm{\Delta}^2 + 2 \tracer{W}{\Delta}.
\end{align*}
Some simple algebra yields
\begin{align*}
\frobnorm{\Delta}^2 & = - 2 \tracer{W}{\Delta} + (R(\Yhatsrl) -
R(Y^*))(R(\Yhatsrl) + R(Y^*)).
\end{align*}

Now, from the definition of the estimate $\Yhatsrl$, we have
\begin{align}
  \label{eq:basic}
R(\Yhatsrl) + \lambda_0 \nucnorm{\Yhatsrl} & \leq R(Y^*) + \lambda
\nucnorm{Y^*}.
\end{align}
Rearranging terms yields
\begin{align*}
R(\Yhatsrl) + R(Y^*) & \leq 2 R(Y^*) + \lambda_0 (\nucnorm{Y^*} -
\nucnorm{\Yhatsrl}) \\
& \stackrel{\1}{\leq} \! 2R(Y^*) \! + \!  \lambda_0 \big (
\nucnorm{P_{Y^*} \Delta} - \nucnorm{P^\perp_{Y^*} \Delta} \big),
\end{align*}
where step $\1$ follows from Lemma~\ref{lem:nuc}, and the fact that
$\lambda > 0$.

Another rearrangement of inequality~\eqref{eq:basic} yields
\begin{align*}
R(\Yhatsrl) - R(Y^*) & \leq \lambda_0 \big( \nucnorm{Y^*} -
\nucnorm{\Yhatsrl} \big) \; \stackrel{\2}{\leq} \lambda_0 \big( 3
\nucnorm{P_{Y^*} \Delta} - \nucnorm{P^\perp_{Y^*} \Delta} \big),
\end{align*}
where step $\2$ follows from Lemma \ref{lem:nuc}, and the fact that
$\nucnorm{P_{Y^*} \Delta} > 0$.  Thus, we have established the upper
bound $(R(\Yhatsrl))^2 - (R(Y^*))^2 \leq T_1 \; T_2$, where
\begin{align*}
T_1 \defn \lambda_0 \left( 3 \nucnorm{P_{Y^*} \Delta} -
\nucnorm{P^\perp_{Y^*} \Delta} \right) \quad \mbox{and} \quad T_2
\defn \left( 2R(Y^*) + \lambda_0 (\nucnorm{P_{Y^*} \Delta} -
\nucnorm{P^\perp_{Y^*} \Delta}) \right).
\end{align*}
Expanding the product of the two terms yields
\begin{multline*}
(R(\Yhatsrl))^2 - (R(Y^*))^2 \leq 6 \lambda_0 R(Y^*) \nucnorm{P_{Y^*}
    \Delta} \!+\! 3 \lambda_0^2 \nucnorm{P_{Y^*} \Delta}^2 \!-\! 2
  \lambda_0 R(Y^*) \nucnorm{P^\perp_{Y^*} \Delta} \\
  + \lambda_0^2 \nucnorm{P^\perp_{Y^*} \Delta}^2 - 4 \lambda_0^2
  \nucnorm{P_{Y^*} \Delta} \: \nucnorm{P^\perp_{Y^*} \Delta} \\
 \stackrel{\3}{\leq} \!\!6 \lambda_0 R(Y^*) \nucnorm{P_{Y^*} \Delta}
 \!+\!  3 \lambda_0^2 \nucnorm{P_{Y^*} \Delta}^2 \!-\! 2 \lambda_0
 R(Y^*) \nucnorm{P^\perp_{Y^*} \Delta},
\end{multline*}
where step $\3$ follows from Lemma \ref{lem:lambda}, since
$\lambda_0^2 \nucnorm{P^\perp_{Y^*} \Delta}^2 - 4 \lambda_0^2
\nucnorm{P^\perp_{Y^*} \Delta} \; \nucnorm{P_{Y^*} \Delta} \leq 0$.

We also note that
\begin{align*}
-2 \tracer{W}{\Delta} & \leq 2 \opnorm{W} \: \nucnorm{\Delta} \!=\! 2
\opnorm{W} (\nucnorm{P_{Y^*} \Delta} + \nucnorm{P^\perp_{Y^*} \Delta
}).
\end{align*}
Combining with the fact that $\lambda_0$ satisfies the inequality $2
\opnorm{W} \leq \lambda_0 R(Y^*)$, we find that
\begin{align*}
\frobnorm{\Delta}^2 & \leq 7 \lambda_0 R(Y^*) \nucnorm{P_{Y^*} \Delta}
+ 2 \lambda_0^2 \nucnorm{P_{Y^*} \Delta}^2 \\
& \stackrel{\4}{\leq} 7 \lambda R(Y^*) \sqrt{\rank(Y^*)}
\frobnorm{\Delta} + 2 \lambda_0^2 \rank(Y^*) \frobnorm{\Delta },
\end{align*}
where in step $\4$, we have used the Cauchy Schwarz inequality and the
fact that projections are non-expansive to write
\begin{align*}
\nucnorm{P_{Y^*} \Delta} & \leq \sqrt{\rank(P_{Y^*} \Delta)} \:
\frobnorm{P_{Y^*} \Delta} \leq \sqrt{\rank(Y^*)} \frobnorm{\Delta}.
\end{align*}
Rearranging yields
\begin{align*}
\frobnorm{\Delta} \left( 1 - 2 \lambda_0^2 \rank(Y^*) \right) \leq 7
\lambda_0 R(Y^*) \sqrt{\rank(Y^*)}.
\end{align*}
Squaring both sides, substituting the choice of $\lambda_0$, and using
the condition \mbox{$\rank(A) \left( \frac{1}{n} + \frac{1}{m} \right)
  \leq 1/20$} completes the proof.
\end{proof}

The only remaining detail is to prove Lemmas~\ref{lem:nuc}
and~\ref{lem:lambda}.


\subsubsection{Proof of Lemma~\ref{lem:nuc}}

We write
\begin{align*}
  \nucnorm{\Yhatsrl} & = \nucnorm{Y^* + \Yhatsrl -Y^*} \\
& = \nucnorm{Y^* - P^\perp_{Y^*} \Delta - P_{Y^*} \Delta} \\
& \geq \nucnorm{Y^* - P^\perp_{Y^*} \Delta} - \nucnorm{P_{Y^*} \Delta} \\
& = \nucnorm{Y^*} + \nucnorm{P^\perp_{Y^*} \Delta} - \nucnorm{ P_{Y^*}
    \Delta}.
\end{align*}
Rearranging yields the claim.  \qed


\subsubsection{Proof of Lemma~\ref{lem:lambda}}

Rearranging the Cauchy Schwarz inequality for two matrices $A$ and $B$
yields
\begin{align*}
\frobnorm{A} - \frobnorm{B} & \geq - \frac{\tracer{B}{B - A}
}{\frobnorm{B}}.
\end{align*}
Now setting $A = Y - \Yhatsrl$ and $B = Y - Y^*$, we have
\begin{align*}
R(\Yhatsrl) - R(Y^*) & \geq - \frac{\tracer{W}{\Yhatsrl -
    Y^*}}{\frobnorm{W}} \\
& \stackrel{\1}{\geq} - \frac{\lambda_0}{2} \opnorm{W} \:
\nucnorm{\Delta},
\end{align*}
where step $\1$ follows from H\"{o}lder's inequality and choice of
$\lambda_0 \geq 2 \frac{\opnorm{W}}{\frobnorm{W}}$.

Combining this with the basic inequality~\eqref{eq:basic} yields
\begin{align*}
\lambda_0 ( \nucnorm{\Yhatsrl} - \nucnorm{Y^*} ) & \leq
\frac{\lambda_0}{2} \nucnorm{\Delta}.
\end{align*}
Finally, using Lemma~\ref{lem:nuc}, we have
\begin{align*}
\nucnorm{P^\perp_{Y^*} \Delta} - \nucnorm{P_{Y^*} \Delta} & \leq
\nucnorm{\Yhatsrl} - \nucnorm{Y^*} \\
& \leq \frac{1}{2} \nucnorm{\Delta} \\
& = \frac{1}{2} \left( \nucnorm{P_{Y^*} \Delta} +
\nucnorm{P^\perp_{Y^*} \Delta} \right),
\end{align*}
which completes the proof.  \qed


\subsection{Proof of Theorem~\ref{thm:algo}}

We write the (reduced) singular value decomposition of a matrix $M$ as
$M = U_M \Sigma_M V_M^\top$. We also adopt the shorthand $r_M =
\rank(M)$ for the rest of this proof.  The {\sf LevSort} algorithm
clearly runs in polynomial time, since it involves a singular value
decomposition and a sorting operation, both of which can be
accomplished efficiently.  Let us now verify the exactness guarantee.

Since the observation model~\eqref{obsmodel} is noiseless and $r_A
\leq r_{X^*}$, we have $r_Y = r_A$.  Moreover, by definition of the
observation model, we have
\begin{align*}
Y^\top Y = \left(X^* \right)^\top A^\top A X^*.
\end{align*}
Consequently, the unknown matrix $X^*$ can be written as
\begin{align*}
X^* = V_A \Sigma_A^{-1} U \Sigma_Y V_Y^\top,
\end{align*}
with $U$ representing an unknown $r_A \times r_A$ unitary matrix
(satisfying $U^\top U= UU^\top = I$). Substituting this representation
of $X^*$ back into the noiseless observation model yields
\begin{align*}
U_Y \Sigma_Y V_Y^{\top} = \Pi^* U_A U \Sigma_Y V_Y^{\top}.
\end{align*}
Now $\Sigma_Y V_Y^\top$ has a full-dimensional row-space, and so we
have $U_Y = \Pi^* U_A U$. We complete the proof by observing that
\begin{align*}
U_Y U_Y^\top = \Pi^* U_A U_A^\top \left(\Pi^* \right)^\top, 
\end{align*}
so that we have the equivalence $\ell(Y) = \Pi^* \ell(A)$ as claimed. 
The uniqueness of the parameters $(\Pi^*, X^*)$ follows from the fact that
the leverage score vectors $\ell(A)$ and $\ell(Y)$ have distinct entries.
\qed


\subsection{Proof of Theorem~\ref{thm:clustering}}

The proofs of Theorems \ref{thm:MLrate}, ~\ref{thm:svt},
and~\ref{thm:srlasso} apply to the model~\eqref{cluster} with minor
modifications. We briefly mention these modifications here, leaving
the details to the reader.

Part (a) follows by mimicking the proof of Section~\ref{sec:thm1ub} as
is, with a small modification to the metric entropy of the observation
space. In particular, the covering number of the observation space is
now upper bounded by $n^n \cdot N (\delta, \mathbb{U}_m^\Pi(A) \cap
\mathbb{B}_F(t), \frobnorm{\cdot} )$, and the rest of the proof
follows as before.

Parts (b) and (c) follow by mimicking the proof of
Sections~\ref{sec:thm2ub} and~\ref{sec:thm3proof}, respectively, with
the definition $Y^* = D^* A X^*$. Note that the clustering observation
model can only decrease the rank of $Y^*$ from before.  \qed

\section{Discussion}
\label{sec:conc}

We conclude with a discussion of some possible future directions.

\subsection{More general picture for regression problems}
Multivariate linear regression is a specific case of the following problem
with shuffled data $\{(a_{\pi(i)}, y_i) \}_{i = 1}^n$, with the covariates $a_i \in
\mathbb{R}^d$ and responses $y_i \in \mathbb{R}^m$ related by the equation
\begin{align}
y_i = f\left(a_{\pi(i)} \right) + w_i,
\end{align}
where $f$ represents a function from some parametric or non-parametric
family $\mathcal{F}$. The general behaviour of prediction error for problems of this
form should be similar to that seen in our linear regression model, or
the structured regression model of Flammarion et
al. \cite{rigollet}. In particular, provided the data $a_i$ is
sufficiently diverse and the function class $\mathcal{F}$ is
sufficiently expressive, the minimax rate of prediction for the
permuted model should be given by the sum of two terms: the minimax rate of the
unpermuted model (or equivalently, with a known permutation), and an additional constant/logarithmic term that accounts for the permutation.

\subsection{Necessity of flatness condition and adaptivity}

Our condition on the matrix $A$ is a convenient one for the
application of the Gilbert-Varshamov type bound on distances between
permuted binary vectors. However, this sufficient condition may be far
from necessary -- we instead require some permutation codes of real
numbers.

Conversely, the upper bound \eqref{upperbound} can be stated by explicitly taking
the structure of the matrix $A$ into account; this will require bounds on
the metric entropy of the union of subspaces generated by permutations of
the range space of $A$.

\bibliographystyle{alpha}

\bibliography{research}

\end{document}